\documentclass{article}%

\usepackage[round]{natbib}

\usepackage{epsfig}
\usepackage{enumerate,graphicx}
\usepackage{amsfonts}
\usepackage{graphicx,color}
\usepackage{amsmath,amsthm,amssymb}
\usepackage{amssymb}

\newtheorem{thm}{Theorem}[section]
\newtheorem{lemma}{Lemma}[section]
\newtheorem{prop}{Proposition}[section]
\newtheorem{cor}{Corollary}[section]
\newtheorem{Defi}{Definition}[section]

\newtheorem{exm}{Example}[section]
\theoremstyle{remark}
\newtheorem{rem}[thm]{\bf Remark}
\newcommand{\eref}[1]{(\ref{#1})}
\allowdisplaybreaks

\begin{document}

%

%

\title{Exponential inequalities for nonstationary Markov Chains}

\author{Pierre Alquier\footnote{CREST, ENSAE, Universit\'{e} Paris Saclay. Pierre Alquier's work has been supported by GENES and  by the French National Research Agency (ANR) under the grant Labex Ecodec (ANR-11-LABEX-0047).}, Paul Doukhan\footnote{AGM UMR8088 University Paris-Seine and CIMFAV, Universidad de Valparaiso, Chile. Paul Doukhan's work  has been developed within the MME-DII center of excellence (ANR-11-LABEX-0023-01) \& PAI-CONICYT MEC $N^o 80170072$.} and Xiequan Fan\footnote{CAM, Tianjin University, Tianjin, China. Fan Xiequan has been partially supported by the National Natural Science Foundation of China (Grant $N^o 11601375$.}}

\maketitle

\begin{abstract}
Exponential inequalities are main tools in machine learning theory. To prove exponential inequalities for non i.i.d random variables allows to extend many learning techniques to these variables. Indeed, much work has been done both on inequalities and learning theory for time series,  in the past 15 years. However, for the non independent case, almost all the results concern stationary time series. This excludes many important applications: for example any series with a periodic behaviour is nonstationary. In this paper, we extend the basic tools of~\cite{DF15} to nonstationary Markov chains. As an application, we provide a Bernstein-type inequality, and we deduce risk bounds for the prediction of periodic autoregressive processes with an unknown period.
\end{abstract}

\section{Introduction}

Exponential inequalities are corner stones of machine learning theory. For example, distribution free generalization bounds were proven by Vapnik and Cervonenkis based on Hoeffding's inequality, see~\cite{Va98}. Model selection bounds in~\cite{M07} also rely on exponential moment inequalities.

To prove such inequalities in the non i.i.d setting is thus crucial to study the generalization ability of machine learning algorithms on time series. As an example, a Bernstein type inequality for $\alpha$-mixing time series is proven in~\cite{MM02}. This result is used by~\cite{SC09} to prove generalization bounds for general learning problems with $\alpha$-mixing observations.

Exponential inequalities and machine learning with non-i.i.d observations actually became an important research direction, a more detailed list of references is given below. However, most of these references assume stationarity. That is, only the independence assumption was removed. The observations are still assumed to be identically distributed, or at least erdogic. This excludes many applications: in addition to trends,  data related to a human activity such as in industry or economics has some periodicity (hourly, daily, yearly\dots) and some regime switching; the same remark applies to  data with a physical origin, such as in geology, astrophysics\dots

In this paper, we generalize the inequalities proven by~\cite{DF15} for time homogeneous Markov chains to non-homogeneous chains. This allows to study a large set of nonstationary processes. We obtain Bernstein and McDiarmid inequalities as well as moments inequalities. As an application, we study periodic autoregressive processes of the form $X_{t} = f_t (X_{t-1}) + \varepsilon_t$ where $f_{t+T}=f_{t}$ for any $t$, for some period $T$. Thanks to our version of Bernstein's inequality we show that the Empirical Risk Minimizer (ERM) leads to consistent predictions in this setting. We also show that a penalized version of the ERM enjoys the same property even when $T$ is unknown.

The paper is organized as follows. The rest of this introduction is dedicated to a state-of-the-art on exponential inequalities for time series. Section~\ref{section:notations} introduces the notations and assumptions that will be used in the whole paper. In Section~\ref{section:theorems}, we state an extension of Proposition 2.1 of~\cite{DF15}: this is Lemma~\ref{McD}. As a proof of concept, we use this lemma to prove a version of Bernstein inequality for nonstationary Markov chains. We also provide Cramer and McDiarmid inequalities based on this lemma. We study periodic autoregressive series in Section~\ref{section:autoregression}. Finally, Section~\ref{section:proof} contains the proof of Lemma~\ref{McD} and of the results in Section~\ref{section:autoregression}.

\subsection{State of the art}

We refer the reader to~\cite{Bo13} for an overview on exponential and concentration inequalities in the i.i.d case. This book also provides references for applications of these results to machine learning theory.

Exponential inequalities were proven for time under a various range of assumptions. We refer the reader to~\cite{D18} for various approaches on modelling time series.

Inequalities for Markov chains $(X_t)_{t\geq 1}$ are proven in~\cite{Cat03,A08,BC10,JO10,Win17,BP18,P18,BC19}. Note that most of these inequalities require the chain to be time homogeneous. While this does not imply the chain to be stationary, in some sense the $X_t$'s are asymptotically identically distributed in these papers. For example, consider the powerful renewal technique used in~\cite{BC19} to prove a version of Bernstein inequality. The proof is based on the fact that blocks $(X_{\tau_i},\dots,X_{\tau_{i+1}-1})$ between two {\it renewal times} $\tau_i$ and $\tau_{i+1}$ are actually i.i.d. It is thus possible to apply the i.i.d version of Bernstein inequality to these blocks. The spectral technique used in~\cite{P18} still relies on the ergodicity of the Markov chain (we thank the anonymous Referee for pointing out some of these references). Exponential inequalities for hidden Markov chains are given in~\cite{KR08}.

It is a well-known fact that Hoeffding's inequality is not only valid for independent observations, but also for martingales increments (it is sometimes refered to as Hoeffding-Azuma inequality in this case). To decompose a function of the process as a sum of martingales increments is actually one of the most powerful techniques to prove exponential inequalities, see Chapter 3 in~\cite{Bo13}. More exponential inequalities for martingales can be found in~\cite{S+12,R13b,BDR15}. This technique is actually used by~\cite{DF15,Fx2} to prove exponential inequalities for Markov chains.

Markov chains are extremely useful in modelisation and simulations, however, many time series have a very different dependence structure. Mixing coefficients allow to quantify the dependence between observations without giving an explicit structure on this dependence. We refer the reader to~\cite{R17} for a comprehensive introduction. Exponential inequalities for mixing processes are proven in~\cite{Sam00,RPR09,R17,HS17}. Mixing series however exclude many stochastic processes, as discussed in the monograph~\cite{DDp07}. Weak dependance coefficients cover a wider range of processes for which Bernstein type inequalities are proven for example in~\cite{CMS,DN07,Win10,RPR10,BZ17}. Dynamic systems are examples of processes where only $X_1$ is random, each $X_t$ is then a deterministic fonction of $X_{t-1}$. Based on weak dependence arguments, it is possible to prove exponential inequalities for such processes~\cite{CMS}.

Based on such inequalities, it is possible to prove generalization bounds for machine learning algorithms~\cite{SHS09,SC09,SK13,LHTG14,HS14,SP15,KM15,MSS17,AB18}. Model selection techniques in the spirit of~\cite{M07} are studied in~\cite{Mei00,L11,AW12}, and aggregation of estimators in~\cite{ALW13}.

In this paper we prove provide tools to prove exponential inequalities for nonstationary, non homogeneous Markov chains. Rather than the renewal or spectral techniques discussed above, we extend the martingale approach of~\cite{DF15} to non-homogeneous chains.

\section{Notations}
\label{section:notations}

From now, all the random variables are defined on a probability space $(\Omega, \mathcal{A}, \mathbb{P})$. Let $({\mathcal X}, d)$ and  $({\mathcal Y}, \delta)$ be two complete separable metric spaces. Let $(\varepsilon_t)_{t \geq 2}$ be a sequence of i.i.d ${\mathcal Y}$-valued random variables. Let $X_1$ be a ${\mathcal X}$-valued random variable independent of $(\varepsilon_t)_{t \geq 2}$. Let $(X_t)_{t \geq 1}$ be the Markov chain given by
\begin{equation}\label{Mchain}
X_t=F_t(X_{t-1}, \varepsilon_t), \quad \text{for $t\geq 2$},
\end{equation}
where the functions $F_t: {\mathcal X} \times {\mathcal Y} \rightarrow {\mathcal X}$ are such that
\begin{equation}\label{contract}
\sup_{t} {\mathbb E}\big [ d\big(F_t(x, \varepsilon_1), F_t(x', \varepsilon_1)\big) \big] \leq
\rho d(x, x')
\end{equation}
for some constant $\rho \in [0,1)$, and
\begin{equation}\label{c2}
 \sup_t d(F_t(x,y), F_t(x,y')) \leq C \delta(y,y')
\end{equation}
for some  constant $C>0$. In particular, when $F_t \equiv F,$ this is the model studied by~\cite{DF15}.

This class of Markov chains, that we call ``one-step contracting", contains a lot
of pertinent examples. The classical AR(1)-process is given by $X_t = F(X_{t-1},\varepsilon_t)$ where $F(x,y) = ax+y$. Condition~\eqref{c2} is satisfied, and Condition~\eqref{contract} will be satisfied as soon as $|a|<1$. Now, consider a time-varying AR(1) process:
$$
X_t = a_t X_{t-1} + \varepsilon_t.
$$
This process may be non-stationary. Condition~\eqref{c2} is still satisfied, and
$
 |F_t(x,y) - F_t(x',y)|  \leq |a_t| |x-x'|
$
so Condition~\eqref{contract} will be satisfied as soon as $\sup_t |a_t| <1$. This process is studied  by~\cite{BD17} under various assumptions: local stationarity, that means a slow variation of $a_t$ as a function of $t$, see~\cite{Da96}, and periodicity, that is, for any $t$: $a_{t+T}=a_t$ for some (known) period $T$. If $T$ is unknown,  a cross-validation procedure to estimate $T$ is  proposed (Remark 2.4) without a consistency result. Below we will propose a penalized procedure with some statistical guarantees.

As a much more general example,  consider  the following functional auto-regressive model. Let ${\mathcal X}$ be a separable Banach space with norm $|\cdot|$. The functional auto-regressive model is defined by
$$
X_t=f_t(X_{t-1}) +  \varepsilon_t \, ,
$$
where $f_t : {\mathcal X}  \rightarrow {\mathcal X}$ is such that
$$
  |f_t(x)-f_t(x')|\leq \rho |x-x'| .
$$
Clearly  (\ref{Mchain}) and (\ref{contract}) are satisfied once $\rho \in [0, 1)$, see \cite{DF99} for more examples.

We introduce the natural filtration of the chain ${\mathcal F}_0=\{\emptyset, \Omega \}$ and for $ t \in \mathbb{N}$, 
$ {\mathcal F}_t = \sigma(X_1, X_2,  \ldots, X_t)$.

Consider a separately Lipschitz function $f: {\mathcal X}^n \to {\mathbb R}$ such that
\begin{equation*}
|f(x_1, \ldots, x_n)-f(x'_1, \ldots, x'_n)| \leq \sum_{t=1}^n d(x_t,x'_t) \, .
\end{equation*}
We define
\begin{equation}\label{Sn}
S_n:=f(X_1, \ldots , X_n) -{\mathbb E}[f(X_1, \ldots , X_n)]\, .
\end{equation}
The objective of what follows will be to derive inequalities on the tails of $ \mathbb{P}(|S_n|>x) $.

\section{Main results}
\label{section:theorems}

\cite{DF15} proved several exponential and moments inequalities if  $F_t \equiv F$, by using  a martingale decomposition. We will first extend this martingale decomposition to the general case. As an example, we will use it to prove a Bernstein type inequality. Other inequalities are given in the appendix. Here $(X_t)_{t \geq 1}$ will be a Markov chain satisfying \eref{Mchain} for some functions $(F_t)_{t \geq 2}$ satisfying \eref{contract}.

\subsection{Main lemma: martingale decomposition}

Set $K_t(\rho)=(1-\rho^{t+1})/(1-\rho)$ 
 for  $t \ge0,\rho \in  [0, 1)$ and
$$
g_k(X_1, \ldots , X_k)= {\mathbb E}[f(X_1, \ldots, X_n)|{\mathcal F}_k]\, ,
$$
and
$$
d_k=g_k(X_1, \ldots, X_k)-g_{k-1}(X_1, \ldots, X_{k-1})\, .
$$
 Define
$
S_t:=d_1+d_2+\cdots + d_t
$, for $t \in [1, n-1]$,
and note that the functional $S_n$ introduced in \eref{Sn} satisfies indeed $ S_n =d_1+d_2+\cdots + d_n$.
Thus $(S_t)$ is a martingale adapted to the filtration ${\mathcal F}_t$, and $(d_t)$ is the martingale difference of $(S_t)$.

Let $P_{X_1}$ denote the distribution of $X_1$ and $P_\varepsilon$ the (common) distribution of the $\varepsilon_t$'s. Let $G_{X_1}$, $G_\varepsilon$ and $H_{t, \varepsilon}$ be defined by
$$
G_{X_1}(x)=\int d(x, x') P_{X_1}(dx'), $$
$$
G_{\varepsilon}(y)= \int C\delta(y,y')P_{\varepsilon}(dy') \text{ and}
$$
$$
H_{t,\varepsilon}(x,y)= \int d(F_t(x,y),F_t(x,y'))P_{\varepsilon}(dy') \, .
$$

We are now in position to state our main lemma.
\begin{lemma}\label{McD} Assume \eqref{Mchain} and  \eqref{contract}, then:
\textcolor{white}{bla}
\begin{enumerate}
\item The function $g_t$ is separately Lipschitz and
$$
|g_t(x_1, \ldots, x_t)-g_t(x'_1, \ldots, x'_t)|
\leq \sum_{i=1}^{t-1} d(x_i,x_i')+ K_{n-t}(\rho) d(x_t, x'_t) \, .
$$
\item The martingale difference $(d_t)$ is such that
\begin{align*}
   |d_1| & \leq K_{n-1}(\rho) G_{X_1}(X_1) \, ,
\\
  \forall t\in[2,n]\text{, } |d_t| & \leq K_{n-t}(\rho)  H_{t,\varepsilon}(X_{t-1}, \varepsilon_t)\, .
\end{align*}
\item
Assume moreover that the $F_t$'s satisfy~\eref{c2}. Then $H_{t,\varepsilon}(x,y) \leq G_\varepsilon(y)$,
and consequently, for $t \in [2,n]$,
$$
 |d_t|\leq K_{n-t}(\rho)  G_\varepsilon(\varepsilon_t)\, .
$$
\end{enumerate}
\end{lemma}
The proof of this lemma is given in Section~\ref{section:proof}. First, we want to show that the inequalities in this lemma can be used to prove exponential inequalities on $S_n$.

\subsection{Application: Bernstein inequality}

Note that~\cite{V95} and~\cite{D99} obtained some tight Bernstein type inequalities for  martingales. Here, we can use the martingale decomposition and apply Lemma~\ref{McD} to obtain the following result.

\begin{thm}\label{Bernsteinineq}
Assume that there exist some constants  $M>0, V_1\geq 0$ and $V_2\geq 0$ such that, for any integer $k\geq 2$,
\begin{equation}
\label{ass:moment}
{\mathbb E} \Big[   G_{X_1}(X_1)^k\Big] \leq \frac {k!}{2} V_1 M^{k-2} \text{,  and }
 {\mathbb E} \Big[   G_\varepsilon(\varepsilon)^k\Big]  \leq
 \frac {k!}{2} V_2 M^{k-2} \, .
\end{equation}
Let $\delta=MK_{n-1}(\rho)$ and
$$V_{(n)} =  V_1\Big(K_{n-1}(\rho)\Big)^2 +  V_2 \sum_{k=2}^n \Big(K_{n-k}(\rho)\Big)^2 .$$
Then, for any $s \in [0, \delta^{-1})$,
\begin{equation}\label{maindfs}
  \mathbb{E}\,[e^ {\pm  sS_n} ]\leq \exp \left (\frac{s^2 V_{(n)} }{2 (1- s\,\delta )} \right )\, .
\end{equation}
Consequently, for any $x> 0$,
\begin{align*}
  {\mathbb P}\big(\pm  S_n \geq  x\big)
  &\leq  \exp \left( \frac{-x^2}{V_{(n)}(1+\sqrt{1+2x \delta/V_{(n)}})+x \delta }\right)\,  \\
 &\leq \exp \left( \frac{-x^2}{2 \,(V_{(n)} +x  \delta ) }\right)\, .
\end{align*}
\end{thm}
The quantity $V_{(n)}$ can be computed explicitely from the definition for each $n$ but note that
\begin{equation}
V_1 + (n-1) V_2 \leq V_{(n)} \leq \frac{V_1 + (n-1) V_2}{(1-\rho)^2}.
\label{order_v}
\end{equation}
\noindent\emph{Proof.} For any $s \in [0, \delta^{-1})$,
\begin{align*}
\mathbb{E}\,[e^{s d_1 } ]
& \leq 1+ \sum_{i=2}^{\infty} \frac{s^i}{i!}\, \mathbb{E}\,[ | d_1|^i ] \\
&\leq 1+ \sum_{i=2}^{\infty} \frac{s^i}{i!}\, \Big(K_{n-1}(\rho)\Big)^i \mathbb{E}\,\Big[ \Big(G_{X_1}(X_1)\Big)^i \Big] \nonumber\\
&\leq 1+ \sum_{i=2}^{\infty} \frac{s^i}{i!}\, \Big(K_{n-1}(\rho)\Big)^i \frac {i!}{2} V_1 M^{i-2} \\
& = 1+ \frac{s^2V_1 \Big(K_{n-1}(\rho)\Big)^2  }{2 (1  -s\, \delta )}\, \\
&\leq  \exp \left (\frac{s^2V_1\Big(K_{n-1}(\rho)\Big)^2}{2 (1  -s\,\delta )}  \right ).
\end{align*}
We use Lemma~\ref{McD} for the second inequality, the moment assumption for  the third one, and the inequality $1+s \leq e^s,$ for the final inequality. Similarly, for any  $k \in [2, n],$
$$
\mathbb{E}\,[e^{s d_k  }|\mathcal{F}_{k-1}] \leq \exp \left( \frac{s^2V_2\Big(K_{n-k}(\rho)\Big)^2}{2 (1  -s\delta)} \right) \, .
$$
By the tower property of conditional expectation, it follows that
\begin{eqnarray*}
\mathbb{E}\,\big[e^{ sS_n} \big] &=&  \mathbb{E}\,\big[ \mathbb{E}\, [e^{ sS_n} |\mathcal{F}_{n-1}  ] \big]\\
&=&  \mathbb{E}\,\big[ e^ { sS_{ n-1}} \mathbb{E}\,  [e^ { sd_n}  |\mathcal{F}_{n-1}  ] \big] \\
&\leq & \mathbb{E}\,\big[ e^ { sS_{ n-1}} \big] \exp \left( \frac{s^2V_2}{2 (1- s\,\delta)}  \right)\\
&\leq &  \exp \left(\frac{s^2 V_{(n)}}{2 (1- s\,\delta)} \right),
\end{eqnarray*}
which gives inequality (\ref{maindfs}). Using the exponential Markov inequality, we deduce that,
for any $x\geq 0$
\begin{eqnarray}
  \mathbb{P}\left( S_{n} \geq x \right)
 &\leq&  \mathbb{E}\, \big[e^{s\,(S_n -x) } \big ] \nonumber\\
 &\leq&   \exp \left(-s\,x  +  \frac{s^2 V_{(n)}}{2 (1- s\,\delta)}    \right)\, . \label{fines}
\end{eqnarray}
Minimizing the right-hand side with respect to $s$ leads to the result. \hfill\qed

\subsection{McDiarmid and Cramer inequalities}

Here, we state other consequences of Lemma~\ref{McD}. However, as our applications are based on Bernstein inequality, we postpone the proof of these results to Section~\ref{section:proof}.

When the Laplace transform of the dominating random variables $G_{X_1}(X_1)$ and $G_{\varepsilon}(\varepsilon_k)$ satisfy the Cram\'{e}r condition, we obtain the following proposition.
\begin{prop}\label{cram}
Assume that there exist some constants $a>0,$ $K_1\geq 1$ and $K_2\geq 1$ such that
$$
{\mathbb E} \Big[ \exp \Big( a G_{X_1}(X_1)\Big)\Big] \leq K_1 $$
and
$$
 {\mathbb E} \Big[ \exp \Big( a   G_\varepsilon(\varepsilon)\Big)\Big] \leq K_2.
$$
Let $$K=\frac{2}{e^{2}} \left( K_1+ K_2\sum_{i=2}^{n} \Big(\frac {K_{n-i}(\rho)}{K_{n-1}(\rho)}\Big)^2 \right) $$
and
$ \delta={a}/{K_{n-1}(\rho)} .$
Then, for any $s \in [0, \delta )$,
$$
  \mathbb{E}\,[e^{\pm  s S_n}]\leq   \exp \left( \frac{s^2 K  \delta^{-2} }{1-s \delta^{-1}}   \right) \, .
$$
Consequently, for any $x> 0$,
\begin{eqnarray*}
  {\mathbb P}\big(\pm  S_n \geq  x\big) &\leq&  \exp \left( \frac{-(x\delta)^2}{2K (1+\sqrt{1+ x \delta / K })+x \delta  }\right)\,    \\
 &\leq&  \exp \left( \frac{-(x\delta)^2}{ 4K+ 2 x  \delta  }\right)\,\, .
\end{eqnarray*}
\end{prop}

Now, consider the case where the increments $d_k$ are bounded.
We shall use an improved version of the well
known inequality by McDiarmid, stated by~\cite{R13}.
For this inequality, we do not assume that \eref{c2} holds.
Thus, Proposition \ref{classic}  applies to any Markov chain $X_{i}=F_i(X_{i-1}, \varepsilon_i)$  for
$F_i$ satisfying \eref{contract}.
Following  \cite{R13},  let
$$\ell(t)= (t - \ln t -1) +t (e^t-1)^{-1} + \ln(1-e^{-t}) \ \ \ \
\textrm{for all}\  t > 0, $$ and let
$$\ell^*(x)=\sup_{t>0}\big(xt- \ell (t) \big) \ \ \ \   \textrm{for all}\   x > 0, $$ be the Young transform of $\ell(t)$.
As quoted by~\cite{R13}, the following inequality holds
\begin{eqnarray*}
\ell^*(x) \geq (x^2-2x) \ln(1-x) \ \ \ \   \textrm{for any}\   x \in [0, 1).
\end{eqnarray*}
Let also $(X_1', (\varepsilon'_i)_{i \geq 2})$ be an independent copy of
$(X_1, (\varepsilon_i)_{i \geq 2})$.
\begin{prop} \label{classic}
Assume that there exist some positive constants $M_k$ such that
$$
\big \|d(X_1,X_1')
\big \|_\infty \leq M_1 $$
and for $k \in [2,   n] $,
$$
\big \|d\big(F_k(X_{k-1}, \varepsilon_k), F_k(X_{k-1}, \varepsilon'_k) \big)
\big \|_\infty \leq M_k .$$
Let
$$
M^2(n,\rho)= \sum_{k=1}^n \big(K_{n-k}(\rho) M_k \big)^2$$
and
$$
D (n, \rho)=\sum_{k=1}^n K_{n-k}(\rho) M_k \, .
$$
Then, for any $s \geq 0$,
\begin{eqnarray}\label{rio1}
\mathbb{E}[e^{\pm   sS_n } ] \ \leq \  \exp \left (  \frac{D^2(n, \rho)}{M^2(n,\rho)}\  \ell \Big( \frac {M^2(n,\rho)\, s} {D(n, \rho)} \Big)\right) \,
\end{eqnarray}
and,  for any   $x \in [0, D(n, \rho)]$,
\begin{eqnarray}\label{rio2}
{\mathbb P}\big(\pm  S_n>x\big) \leq \exp \left ( -\frac{D^2(n, \rho)}{M^2(n,\rho)} \ \ell^*\Big( \frac {x} {D(n, \rho)} \Big)\right).
\end{eqnarray}
Consequently, for any $x \in [0, D(n, \rho)]$,
\begin{eqnarray}\label{riosbound}
{\mathbb P}\big(\pm  S_n>x\big) &\leq&  \left ( \frac{D(n, \rho)-x}{D(n, \rho)}\right)^{ \frac{2D(n, \rho)x -x^2}{M^2(n,\rho)}}\!\!\!.
\end{eqnarray}
\end{prop}

\begin{rem} Since  $(x^2-2x) \ln(1-x) \geq 2\,x^2$,  $\forall x \in [0,1)$, \eqref{riosbound} implies the following McDiarmid inequality
\begin{eqnarray*}
{\mathbb P}\big(\pm  S_n>x\big) \ \leq \  \exp \left ( -\frac{2 x^2}{M^2(n,\rho)}
 \right) \, .
\end{eqnarray*}
\end{rem}

\begin{rem}
Taking  $ \displaystyle  \Delta (n, \rho)=K_{n-1}(\rho)\max_{1\leq k \leq n} M_k$, we obtain, for any $x \in [0, n \Delta(n, \rho)]$,
$$
{\mathbb P}\big(\pm  S_n>x\big)
 \leq \exp \left ( -n \ell^*\Big( \frac {x} {n \Delta(n, \rho)} \Big)\right)
 \leq \exp \left ( -\frac{2 x^2}{n\Delta^2(n,\rho)}
 \right)
  \, .
$$
\end{rem}

\section{Application to periodic autoregressive models}
\label{section:autoregression}

In this section, we apply Theorem~\ref{Bernsteinineq} to predict a nonstationary Markov chain. We will use periodic autoregressive predictors. Of course, these predictors will work well when the Markov chain is indeed periodic autoregressive. However, we will state the results in a more general context -- when the model is wrong, we simply estimate its best prediction by a periodic autoregression.

\subsection{Context}

Let $(X_t)_{t\geq 1}$ be an $\mathbb{R}^d$-valued process defined by the distribution of $X_1$ and, for $t>0$,
$$ X_{t} = f_{t}^*(X_{t-1}) + \varepsilon_t  ,$$
where the $\varepsilon_t$ are i.i.d and centered, and each $f_t^*$ belong to a fixed family of functions $\mathcal{F}$ with $\forall f\in\mathcal{F}$, $\forall (x,y)\in\mathbb{R}^d$, $\|f(x)-f(y)\| \leq \rho \|x-y\|,\ \rho \in [0, 1).$

We are interested by periodic predictors: $f_{t+T}=f_{t}$, defined by a sequence $(f_1,\dots,f_T)\in\mathcal{F}^T$. Of course, if the series $(X_t)$ actually satisfies $f^*_{t+T}=f^*_{t}$, then this family of predictors can give optimal predictions. But they might also perform well when this equality is not exact (for example, when there is a very small drift).

Prediction is assessed with respect to a non-negative loss function: $\ell(\cdot)$. We assume that $\ell$ is $L$-Lipschitz. Note that this includes the absolute loss, the Huber loss and all the quantile losses. This also includes the quadratic loss if we assume that $X_t$, and hence $\varepsilon_t$, is bounded. Given a sample $X_1,\dots,X_n$ we define the empirical risk, for any $f_{1:T}=(f_1,\dots,f_T)\in \mathcal{F}^T$:
$$
 r_n(f_{1:T}) = \frac{1}{n-1}\sum_{i=2}^n \ell\left(X_i - f_{i[T]} (X_{i-1}) \right) ,$$
where $i[T]\in\{1,\ldots,T\}$ is such that $i-i[T] \in T\cdot\mathbb{Z}.$ We then define
$$ R_n(f_{1:T}) = \mathbb{E}[r_n(f_{1:T})]. $$
Note that when the process has actually $T$-periodic distribution, in the sense that the distribution of the vectors $(X_{kT+1},\dots,X_{(k+1)T})$ are the same for any $k$, then alsmot surely $f_t^* = f^*_{t+T}$ for any $t$ and
$$ R_n(f_{1:T}) \xrightarrow[n\rightarrow\infty]{} \frac{1}{T}\sum_{t=1}^T \mathbb{E}[ \ell\left(X_t - f_{t} (X_{t-1}) \right)] $$
the prediction averaged over one period, which appears to be equal to $R_{T+1}(f_{1:T})$. We can actually give a more accurate statement.

\begin{prop}\label{prop-risk}
When the distribution of $(X_{kT+1},\dots,X_{(k+1)T})$ does not depend on $k$,
 $$ | R_{T+1}(f_{1:T}) - R_n(f_{1:T}) | \leq C_0 \frac{T+1}{n-1} ,$$
where $ C_0 = L(1+\rho) \left[ \frac{G_\varepsilon(0)}{1-\rho} + G_{X_1}(0)\right] $.
\end{prop}

(All the proofs are postponed to Section~\ref{section:proof} for the clarity of exposition). The simplest use of Bernstein's inequality is to control the deviation between $r_n(f_{1:T})$ and $R_n(f_{1:T})$ for a fixed predictor $f_{1:T}$.

\begin{cor}
\label{coro-var}
Assume that the moment assumption in Theorem~\ref{Bernsteinineq}, given by~\ref{ass:moment}, is satisfied. Define $V_{(n)}$ and $\delta$ (depending on $M$, $\rho$, $V_1$ and $V_2$) as in Theorem~\ref{Bernsteinineq}.
Then for any $0\leq s < (n-1)/(L(1+\rho)\delta)$,
$$
\mathbb{E}\exp\left( \pm s ( r_n(f_{1:T})- R_n(f_{1:T})) \right)
\leq \exp\left( \frac{s^2(1+\rho)^2 L^2 \frac{V_{(n)}}{n-1}}{2(n-1)-2s(1+\rho)\delta L} \right).
$$
\end{cor}
From~\eqref{order_v} above, we know that
$$
\mathcal{V} := \frac{V_{(n)}}{n-1} \leq \frac{\frac{V_1}{n-1}+V_2}{(1-\rho)^2} \leq \frac{V_1 + V_2}{(1-\rho)^2}
$$
that does not depend on $n$.

\subsection{Estimation with a fixed period}

In this subsection we assume that $T$ is known (we will later show how to deal with the case were it is unknown). Thus, we define the estimator
 $$ \hat{f}_{1:T} = (\hat{f}_1,\dots,\hat{f}_T )
  = \underset{f_{1:T} = (f_1,\dots,f_T)}{\operatorname{argmin}}    r_n(f_{1:T}) .$$

In order to study the statistical performances of $\hat{f}_{1:T}$, a few definitions are in order. For any function $f:\mathbb{R}^d \rightarrow \mathbb{R}^d$ we will use the notation
$$ \|f\|_{\sup} := \sup_{x\neq 0} \frac{\|f(x)\|}{\|x\|} .$$
When considering linear functions, this actually coincides with the operator norm.
\begin{Defi}
 Define the covering number $\mathcal{N}(\mathcal{F},\epsilon)$ as the cardinality of the smallest set $\mathcal{F}_{\epsilon}$ such that $\forall f\in\mathcal{F}$, $\exists f_{\epsilon} \in\mathcal{F}_{\epsilon}$ such that $ \|f-f^\epsilon\|_{\sup}\leq \epsilon $.
 Define the entropy of $\mathcal{F}$ by  $\mathcal{H} (\mathcal{F},\epsilon) = 1\vee \log \mathcal{N}(\mathcal{F},\epsilon)$.
\end{Defi}
Covering numbers are standard tools to measure the complexity of set of predictors in machine learning.
\begin{exm}
 Consider the class of AR(1) predictors $f(x) = a x$, $|a|\leq \rho$. Define $\mathcal{F}_\epsilon$ as the set of all functions $f(x)=i\varepsilon x$ for $i\in \{0,\pm 1,\dots,\pm \lfloor \rho/\epsilon \rfloor \} $. It is clear that $\mathcal{F}_\epsilon$ satisfies the above definition and that ${\rm card}(\mathcal{F}_\epsilon)\leq 1+2\rho/\epsilon \leq 1+2/\epsilon$. Thus, $ \mathcal{N}(\mathcal{F},\epsilon) \leq 1+2/\epsilon$ and so $\mathcal{H} (\mathcal{F},\epsilon) \leq 1\vee \log\left(  1+2 /\epsilon \right) $. In the VAR(1) case, $f(x) = Ax$ where $\|A\|_{\sup} \leq \rho$. Using the set $\mathcal{F}_\epsilon$ of all matrices with entries in $ (\epsilon/\sqrt{d}) \{0,\pm 1,\dots,\pm \lfloor \rho\sqrt{d}/\epsilon \rfloor \} $, we prove that $ \mathcal{N}(\mathcal{F},\epsilon) \leq (1+2\sqrt{d} /\epsilon)^d$ and thus $\mathcal{H} (\mathcal{F},\epsilon) \leq 1\vee d \log(  1+2\sqrt{d}/\epsilon) $.
\end{exm}
We are now in position to state the following result on the convergence of $R_n(\hat{f}_{1:T})$.
\begin{thm}
\label{thm-cvg}
As soon as $ n \geq 1 +  4\delta^2 T \mathcal{H} (\mathcal{F},\frac{1}{Ln}) / \mathcal{V}  $ we have, for any $\eta>0$,
\begin{multline*}
\mathbb{P}\Biggl\{
R_n(\hat{f}_{1:T}) \leq \inf_{f_{1:T}\in\mathcal{F}^T} R_n(f_{1:T})
\\
+
C_1 \sqrt{ \frac{ T\mathcal{H} (\mathcal{F},\frac{1}{Ln})}{n-1} }
+ C_2 \frac{ \log\left(\frac{4}{\eta}\right)}{\sqrt{n-1}}
+ \frac{C_3}{n}
\Biggr\} \geq 1-\eta,
\end{multline*}
where $C_1 = 4(1+\rho) L\sqrt{\mathcal{V}}$, $C_2  = 2(1+\rho)L\sqrt{\mathcal{V}} + 2\delta$ and $C_3 = 3[G_{\varepsilon}(0)+G_{X_1}(0)]/(1-\rho) + \mathcal{V}/(2\delta)$.
\end{thm}
The theorem states that the predictor $\hat{f}_{1:T}$ predict as well as the best possible one up to an estimation error that vanishes at rate $\sqrt{n}$. For example, using (periodic) VAR(1) predictors in dimension $d,$ we get a bound in
$$ R_n(\hat{f}_{1:T}) \leq \inf_{f_{1:T}\in\mathcal{F}^T} R_n(f_{1:T}) + \mathcal{O}\left( d \sqrt{\frac{T\log(nd)}{n}} \right). $$

\begin{rem}
When the series is indeed stationary for a known $T$, it is to be noted that $(X_{iT+1},\dots,X_{i(T+1)})_{i\geq 0}$ is a time homogeneous Markov chain. In this case, our technique is not really necessary: it would be possible to apply the inequality from~\cite{DF15}. However, when $T$ is not known, this becomes impossible. In this case, one has to compare the empirical risks of $\hat{f}_{1:T}$ for the various possible $T$'s, and for most of them, $(X_{iT+1},\dots,X_{i(T+1)})_{i\geq 0}$ is not homogeneous. In this case, vectorization cannot help. On the other hand, our inequality can be used for period selection, as detailed in the next subsection.
\end{rem}

\subsection{Period and model selection}

We define a penalized estimator in the spirit of~\cite{M07}. Fix a maximal period $T_{\max}$, for example $T_{\max}=\lfloor n/2 \rfloor$. We propose the following penalized estimator for $T$:
$$ \hat{T} = \arg\min_{1\leq T \leq T_{\max}} \left[ r_n(\hat{f}_{1:T}) +
\frac{C_1}{2} \sqrt{ \frac{ T\mathcal{H} (\mathcal{F},\frac{1}{Ln})}{n-1} }\, \right]. $$
Using this estimator, we have the following result.
\begin{thm} \label{thm-selection}
 For any $\eta>0$ we have,
\begin{multline*}
\mathbb{P}\Biggl\{
R_n(\hat{f}_{1:\hat{T}}) \leq \inf_{1\leq T \leq T_{\max}}\  \inf_{f_{1:T}\in\mathcal{F}^T}
\Biggl[ R_n(f_{1:T})
\\
+
C_1 \sqrt{ \frac{ T\mathcal{H} (\mathcal{F},\frac{1}{Ln})}{n-1} }
+ C_2 \frac{ \log\left(\frac{4 T_{\max}}{\eta}\right)}{\sqrt{n-1}}
+ \frac{C_3}{n} \Biggr]
\Biggr\} \geq 1-\eta,
\end{multline*}
as soon as $ n \geq 1 +  4\delta^2  T_{\max} \mathcal{H} (\mathcal{F},\frac{1}{Ln}) / \mathcal{V}   $.
\end{thm}
Note that $\hat{T}$ depends on $C_1=4(1+\rho) L\sqrt{\mathcal{V}} \leq 4(1+\rho) L \sqrt{V_1 + V_2}/(1-\rho) $. While $L$ depends only on the loss that is chosen by the statistician, in many applications $\rho$, $V_1$ and $V_2$ are unknown. We recommend to use an empirical criterion like the slope heuristic, introduced by~\cite{BM06}, to calibrate $C_1$. This procedure is as follows:
\begin{enumerate}
 \item define, for any $c>0$, $ \hat{T}(c) = \arg\min_{1\leq T \leq T_{\max}} [ r_n(\hat{f}_{1:T}) +
c  \sqrt{ T} ]$.
 \item fix a small step $\epsilon>0$ and define $\hat{c}$ as the maximiser of the jump $J_\epsilon(c)=\sqrt{\hat{T}(c+\epsilon)}-\sqrt{\hat{T}(c)}$.
 \item select $\hat{T}(2\hat{c})$.
\end{enumerate}
Many variants, details on fast implementations and references for theoretical results (in the i.i.d case) can be found in see~\cite{BMM12}.
A theoretical study of the slope heuristic in the context could be the object of future works.

\subsection{Simulation study}

As an illustration we simulate $ X_{t+1} = a_t X_t + \varepsilon_t $
for $t=1,\dots,400$, where $a_{t+4}=a_t$, $(a_1,a_2,a_3,a_4) = (0.8,0.5,0.9,-0.7)$ and $\varepsilon_t\sim\mathcal{N}(0,1)$. The data is shown in Figure~\ref{figure-data} and the autocorrelation function in Figure~\ref{figure-acf}. It is clear that a statistician trying to estimate an AR(1) model with a fixed coefficient would be puzzled by this situation.
\begin{figure}[h]
\vspace{-.5truecm}
\begin{center}
\includegraphics[scale=0.37]{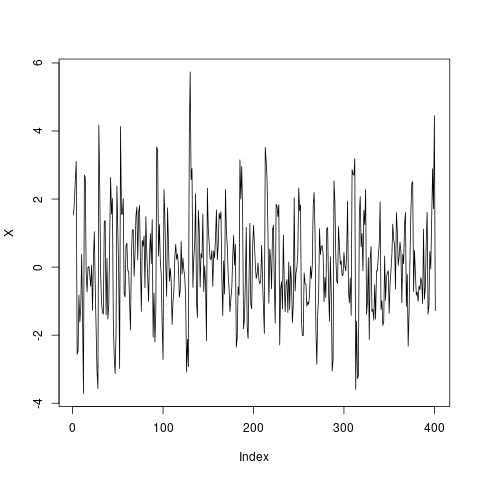}
\end{center}
\vspace{-.5truecm}
\caption{Simulated data.}
\label{figure-data}
\end{figure}
\begin{figure}[h]
\begin{center}
\includegraphics[scale=0.37]{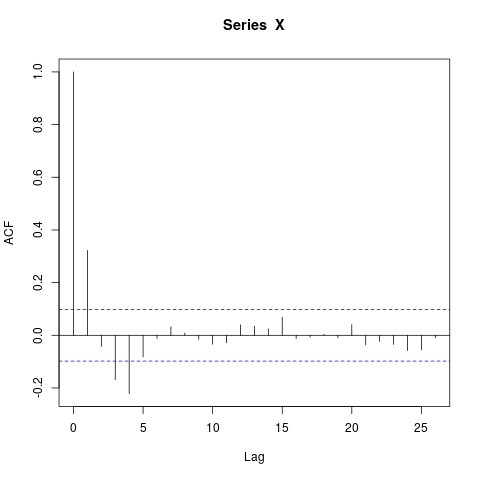}
\end{center}
\vspace{-.5truecm}
\caption{Autocorrelation function of the data.}
\label{figure-acf}
\end{figure}

The dependence of $r_n(\hat{a}_{1:T})$ with respect to $T\in\{1,...,T_{\max}\}$ with $T_{\max}=20$ is shown in Figure~\ref{figure-erm}.
\begin{figure}[h]
\vspace{-.5truecm}
\begin{center}
\includegraphics[scale=0.37]{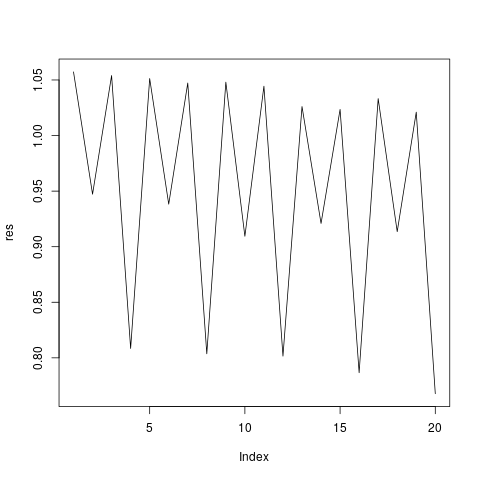}
\end{center}
\vspace{-.5truecm}
\caption{The empirical risk as a function of $T$.}
\label{figure-erm}
\end{figure}
The choice $T=4$ leads to an improvement with respect to $T<4$. On the other hand, we observe a slow linear decrease of $r_n(\hat{a}_{1:4})$, $r_n(\hat{a}_{1:8})$, $r_n(\hat{a}_{1:12})$ \dots this is a sign of overfitting. And indeed,
\begin{enumerate}
 \item for $c< 0.008 $, $\hat{T}(c)=20$,
 \item for $0.009 < c< 0.239 $, $\hat{T}(c)=4$,
 \item for $ 0.240 < c$, $\hat{T}(c)=1$.
\end{enumerate}
Thus, $\hat{c}\simeq 0.0085$ and we choose $\hat{T}(2\hat{c})=\hat{T}(0.017)=4$.

\subsection{Acknowledgements}

We thank the anonymous Referees for their very constructive comments that helped to improve the clarity of the paper.

\section{Proofs}
\label{section:proof}

\begin{proof}[Proof of Lemma~\ref{McD}] The first point will be proved by backward induction. The result is obvious for $t=n$, since $g_n=f$. Assume that it is true at step $t$, and let us prove it
at step $t-1$. By definition
\begin{multline*}
 g_{t-1}(X_1, \ldots, X_{t-1})
 ={\mathbb E}[g_t(X_1, \ldots, X_t)|{\mathcal F}_{t-1}] \\
 = \int g_t(X_t, \ldots, X_{t-1}, F_t(X_{t-1},y)) P_{\varepsilon}(dy)\, .
\end{multline*}
It follows that
\begin{multline}\label{triv1}
|g_{t-1}(x_1,  \ldots, x_{t-1})-g_{t-1}(x'_1,  \ldots, x'_{t-1})|\\ \leq \int |g_t(x_1, \ldots,x_{t-1}, F_t(x_{t-1},y))
\\
-g_t(x'_1, \ldots,x'_{t-1}, F_t(x'_{t-1},y))|
P_{\varepsilon}(dy)\, .
\end{multline}
Now, by assumption and condition \eref{contract},
\begin{align*}\int |g_t(x_1, \ldots, F_t(x_{t-1},y))
& -g_t(x'_1, \ldots, F_t(x'_{t-1},y))|
P_{\varepsilon}(dy) \nonumber \\
&\leq d(x_1,x'_1)+\cdots + d(x_{t-1}, x'_{t-1})
\\
& + K_{n-t}(\rho) \int d(F_t(x_{t-1},y), F_t(x'_{t-1},y))P_{\varepsilon}(dy)
 \nonumber \\
&\leq d(x_1,x'_1)+\cdots  + (1+\rho K_{n-t}(\rho)) d(x_{t-1}, x'_{t-1}) \nonumber \\
&\leq d(x_1,x'_1)+\cdots  +  K_{n-t+1}(\rho) d(x_{t-1}, x'_{t-1}) \, .
\end{align*}
Point 1 follows from this last inequality and~\eqref{triv1}.

Let us now prove Point 2. First note that
\begin{eqnarray}
|d_1|&=&\Big|g_1(X_1)-\int g_1(x)P_{X_1}(dx)\Big| \nonumber \\
&\leq&  K_{n-1}(\rho)\int d(X_1,x)P_{X_1}(dx) \nonumber \\
&=& K_{n-1}(\rho) G_{X_1}(X_1) \,
\end{eqnarray}
where the inequality comes from the first point of Lemma~\ref{McD}.
In the same way, for $t \geq 2$,
\begin{align*}
 |d_t| & = \big|g_t(X_1, \cdots, X_t)-{\mathbb E}[g_t(X_1, \cdots, X_t)|{\mathcal F}_{t-1}]\big|\\
 &\leq  \int \big |g_t(X_1, \cdots, F_t(X_{t-1}, \varepsilon_t))
  \\
  & \quad \quad -g_t(X_1, \cdots, F_t(X_{t-1}, y))\big| P_{\varepsilon}(dy)\\
 &\leq K_{n-t}(\rho)\int  d(F_t(X_{t-1},\varepsilon_t),
 F_t(X_{t-1}, y))P_{\varepsilon}(dy)\\
 &  = K_{n-t}(\rho)
  H_{t,\varepsilon}(X_{t-1},\varepsilon_t) \, .
\end{align*}

Finally, the proof of Point 3 is direct: if \eref{c2} is true, then
$$
H_{t,\varepsilon}(x,y)= \int d(F_t(x,y),F_t(x,y'))P_{\varepsilon}(dy')
\leq \int C\delta(y,y')P_{\varepsilon}(dy')= G_\varepsilon(y) \, .
$$
The proof of the proposition is now complete.
\end{proof}

We state a lemma that will be used in the following proofs.
\begin{lemma} \label{lemma_moments}
Under the assumptions of Section~\ref{section:autoregression} we have
$$ \forall n\in\mathbb{N}\setminus\{0\}\text{, }
\mathbb{E} \|X_n\|
\leq \frac{G_\varepsilon(0)}{1-\rho} + \rho^{n-1} G_{X_1}(0).
$$
\end{lemma}

{\it Proof of Lemma~\ref{lemma_moments}}
 By definition of $G_{X_1}$, $ \mathbb{E} \|X_1\|  =  \int \|x-0\| {\rm d}P_{X_1}(x) = G_{X_1}(0)  $. Then,
$$
 \mathbb{E} \|X_n\|   = \mathbb{E} \|f_n(X_{n-1}) + \varepsilon \|  
  \leq \mathbb{E} \|f_n(X_{n-1})\|  + \mathbb{E} \|\varepsilon \|  
  \leq \rho \mathbb{E} \|X_{n-1}\|  + G_\varepsilon(0).
$$
 So, by induction, for $n>1$,
\begin{align*}
\mathbb{E} \|X_n\|
& \leq (1+\rho+\dots+\rho^{n-2})G_\varepsilon(0) + \rho^{n-1} G_{X_1}(0) \\
& \leq \frac{G_\varepsilon(0)}{1-\rho} + \rho^{n-1} G_{X_1}(0).\qquad\qquad \qquad\qquad\square
 \end{align*}

\noindent\emph{Proof of Proposition~\ref{cram}.} Let $\delta=a/K_{n-1}(\rho).$  Since $\mathbb{E}\,[  d_1 ] =0$, it follows that, for any $s \in   [0, \delta )$,
\begin{eqnarray}\label{finsa3f}
\mathbb{E}\,[e^{s d_1 } ] &=& 1+ \sum_{j=2}^{\infty} \frac{s^j}{j!}\, \mathbb{E}\,[   (d_1)^j ] \nonumber\\
&\leq& 1+ \sum_{i=2}^{\infty} \Big(\frac{s}{\delta} \Big)^j \, \mathbb{E}\,\Big[\frac{1}{j!} | \delta  d_1|^j \Big].
\end{eqnarray}
Note that, for $s\geq 0$,
\begin{align}\label{constantsecond}
\frac{s^j}{j !} e^{-s} &\leq \frac{j^j}{j !} e^{-j} 
\leq 2 e^{-2}, \quad \text{for any $j\geq 2$,}
\end{align}
where the last inequality follows from the fact that $j^j e^{-j}/j!$ is decreasing in $j$. Notice that the equality in \eref{constantsecond}
 is reached at $s=j=2$. By \eref{constantsecond}, Lemma~\ref{McD} and the hypothesis of the proposition, we deduce  that
\begin{eqnarray}\label{fisa3f}
\mathbb{E}\,\Big[\frac{1}{j!} |\delta d_1|^j \Big]   &\leq& \ 2e^{-2}  \mathbb{E}\, [e^{\delta |d_1|}  ] \nonumber\\
 &\leq& \ 2e^{-2} {\mathbb E} \Big[ \exp \Big( a G_{X_1}(X_1)\Big)\Big] \nonumber\\
 &\leq& \ 2e^{-2} K_1.
\end{eqnarray}
Combining Inequalities (\ref{finsa3f}) and (\ref{fisa3f}),  we get, for any $s \in [0, \delta )$,
$$
\mathbb{E}\,[e^{s d_1 } ] \ \leq\ 1+ \sum_{j=2}^{\infty} \frac{2}{e^{2}}\Big(\frac{s}{ \delta } \Big)^j  K_1
\ =\ 1+  \frac{2}{e^{2}} \frac{s^2 K_1 \delta^{-2} }{1-s \delta^{-1}}  
\ \leq\  \exp \left( \frac{2}{e^{2}} \frac{s^2 K_1 \delta^{-2} }{1-s \delta^{-1}}   \right).
$$
In the same way, for $n\geq i>1$, using Lemma~\ref{McD},
\begin{eqnarray*}
\mathbb{E}\,[e^{s d_i }|\mathcal{F}_{i-1} ] &=& 1+ \sum_{j=2}^{\infty} \Big(\frac{s}{\delta} \Big)^j  \, \mathbb{E}\,\Big[\frac{1}{j!} | \delta  d_i|^j |\mathcal{F}_{i-1}  \Big]
\\
&\leq& 1+ \sum_{j=2}^{\infty} \left(\frac{s}{\delta} K_{n-i}(\rho) \right)^j  \, \mathbb{E}\,\Big[\frac{1}{j!} | \delta  G_{\varepsilon}(\varepsilon_i)|^j |\mathcal{F}_{i-1}\Bigr]
\\
&\leq& 1+ \sum_{j=2}^{\infty} \left(\frac{s}{\delta} \frac{K_{n-i}(\rho)}{K_{n-1}(\rho)} \right)^j  \, \mathbb{E}\,\Big[\frac{1}{j!} | a  G_{\varepsilon}(\varepsilon_i)|^j\Bigr]
\\
&\leq& 1+ \left(\frac{K_{n-i}(\rho)}{K_{n-1}(\rho)}\right)^{2} \sum_{j=2}^{\infty} \left(\frac{s}{\delta}  \right)^j \, 2e^{-2}K_2
\end{eqnarray*}
where we used~\eqref{constantsecond} again, and the fact that $K_{n-i}(\rho)/K_{n-1}(\rho)\leq 1$ for any $i \in [2, n]$ which implies $[K_{n-i}(\rho)/K_{n-1}(\rho)]^j\leq [K_{n-i}(\rho)/K_{n-1}(\rho)]^2$ for any $j\geq 2$. So, for any $s \in [0,   \delta )$,
\begin{eqnarray*}
  \mathbb{E}\,[e^{s d_i  }|\mathcal{F}_{i-1}] \  &\leq& 
   1+  \frac{2}{e^{2}} \frac{s^2 K_2 \delta^{-2} }{1-s \delta^{-1}}   \left(\frac{K_{n-i}(\rho)}{K_{n-1}(\rho)}\right)^{2}
  \\
  &\leq& \ \exp \left(\frac{2}{e^{2}} \frac{s^2 K_2 \delta^{-2} }{1-s \delta^{-1}} \Big(\frac {K_{n-i}(\rho)}{K_{n-1}(\rho)}\Big)^2  \right) \, .
\end{eqnarray*}
Using the tower property of conditional expectation,  we have, for any $s \in [0, \delta )$,
\begin{eqnarray*}
\mathbb{E}\,\big[e^{ sS_i} \big] &=&  \mathbb{E}\,\big[ \,\mathbb{E}\, [e^{ s S_i} |\mathcal{F}_{i-1}  ] \big] \nonumber\\
&=&  \mathbb{E}\,\big[ e^ { s S_{i-1}} \mathbb{E}\,  [e^ { s d_i}  |\mathcal{F}_{i-1}  ] \big] \nonumber \\
&\leq & \mathbb{E}\,\big[  e^ { s S_{i-1}}  \big] \exp \left( \frac{2}{e^{2}} \frac{t^2 K_2 \delta^{-2} }{1-t \delta^{-1}}   \Big(\frac {K_{n-i}(\rho)}{K_{n-1}(\rho)}\Big)^2 \right).
\end{eqnarray*}
By recursion,
\begin{align*}
 \mathbb{E}\,\big[e^{ sS_i} \big] & \leq
  \mathbb{E}\,\big[  e^ { s S_{1}}  \big] \exp \left( \frac{2}{e^{2}} \frac{t^2 K_2 \delta^{-2} }{1-t \delta^{-1}} \sum_{i=2}^n  \Big(\frac {K_{n-i}(\rho)}{K_{n-1}(\rho)}\Big)^2 \right) \\
  & \leq   \ \exp \left( \frac{2}{e^{2}} \frac{s^2 K_1 \delta^{-2} }{1-s \delta^{-1}}   \right)\exp \left( \frac{2}{e^{2}} \frac{t^2 K_2 \delta^{-2} }{1-t \delta^{-1}} \sum_{i=2}^n  \Big(\frac {K_{n-i}(\rho)}{K_{n-1}(\rho)}\Big)^2 \right) \\
  & =
  \exp \left( \frac{s^2 K  \delta^{-2} }{1-s \delta^{-1}}   \right),
\end{align*}
where $$K=\frac{2}{e^{2}} \bigg( K_1+ K_2\sum_{i=2}^{n} \Big(\frac {K_{n-i}(\rho)}{K_{n-1}(\rho)}\Big)^2 \bigg).$$
Then using the exponential Markov  inequality, we deduce that,
for any $x\geq 0$ and $s \in [0, \delta )$,
\begin{eqnarray*}
  \mathbb{P}\left( S_{n} \geq x \right)
 &\leq&  \mathbb{E}\, [e^{s\,(S_n -x) }  ] \nonumber\\
 &\leq&   \exp \left(-s x  +  \frac{s^2 K  \delta^{-2} }{1-s \delta^{-1}}   \right)\, .
\end{eqnarray*}
The  minimum is reached at $$s=s(x):= \frac{x\delta^2/K}{x\delta/K + 1 + \sqrt{1+x\delta/K}} .$$ The proposition is proven. \hfill\qed

\noindent\emph{Proof of Proposition~\ref{classic}.}
Denote
$$
u_{k-1}(x_1, \ldots, x_{k-1})= \text{ess}\inf_{\varepsilon_k} g_k(x_1, \ldots, F_k(x_{k-1}, \varepsilon_k))
$$
and
$$
v_{k-1}(x_1, \ldots, x_{k-1})= \text{ess}\sup_{\varepsilon_k} g_k(x_1, \ldots, F_k(x_{k-1}, \varepsilon_k)).
$$
From the proof of Lemma~\ref{McD},  it is easy to see that
$$
u_{k-1}(X_1, \ldots, X_{k-1}) \leq d_k \leq v_{k-1}(X_1, \ldots, X_{k-1})\, .
$$
By Lemma~\ref{McD} and the hypothesis of the proposition, it follows that
$$
v_{k-1}(X_1, \ldots, X_{k-1})-u_{k-1}(X_1, \ldots, X_{k-1})
\leq K_{n-k}(\rho) M_k\, .
$$
Following exactly the proof of  Theorem 3.1 of \cite{R13} with $\Delta_k =  K_{n-k}(\rho) M_k$,
we get~\eqref{rio1} and~\eqref{rio2}.  Since
$\ell^*(x) \geq (x^2-2x) \ln(1-x)$ for any $x \in [0, 1]$, \eqref{riosbound}
follows from~\eqref{rio2}. \hfill\qed

\begin{proof}[Proof of Proposition~\ref{prop-risk}]
 Put $k=\lfloor (n-2)/T \rfloor$, and consider the three possible cases $kT\leq n-2 \leq (k+1/2)T$, $(k+1/2)T< n-2 < (k+1)T-1$ and $n-2 = (k+1)T-1$.
 
\noindent {\it First case}. We write:
 \begin{multline}
 \label{split31}
 R_n(f_{1:T})  =  \mathbb{E}\left[ \frac{1}{n-1}\sum_{i=2}^n \ell\left(X_i - f_{i[T]} (X_{i-1}) \right)\right]
 \\
  = \frac{ kT R_{T+1}(f_{1:T})}{n-1} +  \mathbb{E}\left[ \frac{1}{n-1}\sum_{i=kT+2}^n \ell\left(X_i - f_{i[T]} (X_{i-1}) \right)\right].
 \end{multline}
First,
\begin{align*}
 \left| R_n(f_{1:T})-\frac{ kT R_{T+1}(f_{1:T})}{n-1}  \right|
& = \frac{1}{n-1} \sum_{i=kT+2}^{n} \mathbb{E}\left[ | \ell\left(X_i - f_{i[T]} (X_{i-1}) \right)| \right]
\\
& \leq \frac{1}{n-1} \sum_{i=kT+2}^{n} L \mathbb{E} \|X_i\|  + \rho L \mathbb{E} \|X_{i-1}\|
\\
& \leq  \frac{(n-2)-kT }{n-1} L(1+\rho) \left[ \frac{G_\varepsilon(0)}{1-\rho} + G_{X_1}(0)\right]
\end{align*}
where we used Lemma~\ref{lemma_moments} and $\rho^{n-1}<1$ for the last inequality.
In the same way,
\begin{align*}
\left| \frac{ kT R_{T+1}(f_{1:T})}{n-1} - R_{T+1}(f_{1:T})\right| & = \left(1-\frac{kT}{n-1} \right)R_{T+1}(f_{1:T}) \\
& =  \frac{(n-1)-kT}{n-1} \frac{1}{T}\sum_{i=2}^{T+1}  \mathbb{E}\left[ | \ell\left(X_i - f_{i} (X_{i-1}) \right)| \right]\\
& \leq  \frac{(n-2)-kT+1}{n-1} L(1+\rho) \left[ \frac{G_\varepsilon(0)}{1-\rho} + G_{X_1}(0)\right].
\end{align*}
Combining both inequalities leads to:
\begin{align*}
\left| R_n(f_{1:T})-R_{T+1}(f_{1:T}) \right|
&
\leq \frac{2[(n-2)-kT]+1}{n-1}L(1+\rho) \left[ \frac{G_\varepsilon(0)}{1-\rho} + G_{X_1}(0)\right]
\\
& \leq \frac{T+1}{n-1} C_0
\end{align*}
by definition of $C_0$ and as in the first case, $(n-2)-kT \leq T/2$.

\noindent {\it Case 2}. We write the decomposition of $R_n(f_{1:T})$ in a different way from what we did in~\eqref{split31}:
 \begin{multline*}
 R_n(f_{1:T})  =  \mathbb{E}\left[ \frac{1}{n-1}\sum_{i=2}^n \ell\left(X_i - f_{i[T]} (X_{i-1}) \right)\right]
 \\
  = \frac{ (k+1)T R_{T+1}(f_{1:T})}{n-1} -  \mathbb{E}\left[ \frac{1}{n-1}\sum_{i=n+1}^{(k+1)T+1} \ell\left(X_i - f_{i[T]} (X_{i-1}) \right)\right].
 \end{multline*}
Similar derivations lead to
$$
 \left| R_n(f_{1:T})-\frac{ (k+1)T R_{T+1}(f_{1:T})}{n-1}  \right|
 \leq  \frac{(k+1)T+1-n}{n-1} C_0
$$
and
$$
\left| \frac{ (k+1)T R_{T+1}(f_{1:T})}{n-1} - R_{T+1}(f_{1:T})\right|
 \leq  \frac{(k+1)T+1-n}{n-1}C_0,
$$
and combining both inequalities and $(k+1/2)T< n-2$ gives:
$$
\left| R_n(f_{1:T})- R_{T+1}(f_{1:T}) \right| \leq \frac{T}{n-1} C_0.
$$

\noindent {\it Case 3}. In this case,
$$ R_n(f_{1:T}) = \frac{(k+1)T R_{T+1}(f_{1:T})}{n-1} = R_{T+1}(f_{1:T}) $$
and so
$$ |R_n(f_{1:T}) - R_{T+1}(f_{1:T}) | = 0 .$$

So, in Cases 1, 2 and 3, the largest bound
$$\left| R_n(f_{1:T})-R_{T+1}(f_{1:T}) \right|
\leq \frac{T+1}{n-1}C_0 $$
holds\footnote{We thank one of the anonymous Referees who suggested improvements in this proof that led to the bound $\frac{T+1}{n-1}C_0$, instead of the bound $\frac{2T+1}{n-1}C_0$ stated in the first version of this paper.}.
\end{proof}

\noindent\emph{Proof of Corollary~\ref{coro-var}.}
By definition, we have $ X_t = F_t(X_{t-1},\varepsilon_t) = f_{t}^*(X_{t-1}) + \varepsilon_t$.
So $\| F_t(x,\varepsilon_t)-F_t(x',\varepsilon_t) \| = \| f_{t}^*(x) - f_{t}^*(x') \| \leq \rho \|x-x'\| $ and so~\eqref{contract} is satisfied, and $\|F_t(x,y)-F_t(x,y')\|=\|y-y'\|$ so that~\eqref{c2} is satisfied with $C=1$. We consider the random variable $S_n = f(X_1,\dots,X_n)-\mathbb{E}[f(X_1,\dots,X_n)]$, where
$$ f(x_1,\dots,x_n) = \frac{1}{L(\rho+1)} \sum_{t=2}^n \ell\left(x_t - f_{t[T]}(x_{t-1}) \right). $$
Remark that
\begin{align*}
& | f(x_1,\dots,x_n) - f(x_1,\dots,x_t',\dots,x_n) | \\
    & \leq \frac{\left| \ell \left(x_{t+1} - f_{(t+1)[T]}(x_{t}) \right) - \ell \left(x_{t+1} - f_{(t+1)[T]}(x_{t}') \right)\right|}{L(\rho+1)}
    \\
      & \quad \quad  + \frac{ \left|  \ell \left(x_t - f_{t[T]}(x_{t-1}) \right) - \ell \left(x_t' - f_{t[T]}(x_{t-1}) \right) \right|}{L(\rho +1)}
        \\
   & \leq \frac{\|f_{(t+1)[T]}(x_{t})-f_{(t+1)[T]}(x_{t}')\| + \|x_t-x_t'\|}{\rho+1} \\
   & \leq \|x_t - x_t'\|.
\end{align*}
So the assumptions of Theorem~\ref{Bernsteinineq} are satisfied and
$$ \mathbb{E}\exp\left( \pm t S_n \right) \leq \exp\left(\frac{t^2 V_{(n)}}{2-2t\delta }\right) .$$
Remind that $S_n = \frac{n-1}{L(1+\rho)}\left[ r_n(f_{1:T})- \mathbb{E}[r_n(f_{1:T})]\right] $, and $ R_n(f_{1:T}) = \mathbb{E}[r_n(f_{1:T})] $, so
that by setting $s= {t(n-1)}/{L(1+\rho)}$ we  end the proof.
\hfill \qed

\noindent\emph{Proof of Theorem~\ref{thm-cvg}.}
Fix $\epsilon > 0$. We have, for any $f_{1:T}\in\mathcal{F}_{\epsilon}$, the deviation inequality from Corollary~\ref{coro-var}. A union bound on $f_{1:T}\in\mathcal{F}_{\epsilon}$ leads to, for any $s\in\left[0,\frac{n-1}{L(1+\rho)\delta}\right)$,
\begin{align*}
& \mathbb{P}\left(\sup_{(f_{1:T})\in\mathcal{F}_{\epsilon}^T}|r_n(f_{1:T})-R_n(f_{1:T})| > x \right) \\
& \leq \sum_{(f_{1:T})\in\mathcal{F}_{\epsilon}} \mathbb{P}\left(|r_n(f_{1:T})-R_n(f_{1:T})| > x \right)
\\
& \leq \sum_{(f_{1:T})\in\mathcal{F}_{\epsilon}^T} \mathbb{E}\exp\left( s |r_n(f_{1:T})-R_n(f_{1:T})| - s x\right)
\\
& \leq 2 \mathcal{N}(\mathcal{F} ,\epsilon)^T \exp\left( \frac{s^2(1+\rho)^2 L^2 \mathcal{V}}{2(n-1)-2s(1+\rho)\delta L} -s x \right).
\end{align*}
Now, for any $f_{1:T}\in\mathcal{F}^T$ we construct $f_{1:T}^{\epsilon}=(f_1^\epsilon,\dots,f_T^{\epsilon})$ by chosing, for any $t\in\{1,\dots,T\}$, a function $f_t^\epsilon$ such that $\|f_t-f_t^\epsilon\|_{\sup}\leq \epsilon$, as allowed from the definition of $\mathcal{F}_{\epsilon} $.
Obviously
\begin{multline*}
\left| \ell(X_t-f^\epsilon_{t[T]} (X_{t-1})) -  \ell(X_t - f_{t[T]} (X_{t-1}) )\right|
\\
\leq L \|f^\epsilon_{t[T]} (X_{t-1})   - f_{t[T]} (X_{t-1})\|
\leq L\epsilon \|X_{t-1}\|
\end{multline*}
and as a consequence,
$$|r_n(f_{1:T})-r_n(f_{1:T}^\epsilon)| \leq   L\epsilon\cdot\frac{\sum_{t=1}^{n-1}\|X_t\|}{n-1},$$
and
$$|R_n(f_{1:T})-R_n(f_{1:T}^\epsilon)| \leq \epsilon L\cdot\frac{\sum_{t=1}^{n-1}\mathbb{E}\|X_t\|}{n-1} . $$
Using Theorem~\ref{Bernsteinineq} with $f(X_1,\dots,X_n)=\sum_{t=1}^{n-1} \|X_t\|$ we have, for any $y>0$,
\begin{align*}
& \mathbb{P}\left(\sum_{t=1}^{n-1}\|X_t\| > \sum_{t=1}^{n-1}\mathbb{E}\|X_t\| + y \right) \\
& \leq \mathbb{E}\exp\left[\frac{1}{2\delta}\left(\sum_{t=1}^{n-1}\|X_t\| - \sum_{t=1}^{n-1}\mathbb{E}\|X_t\| - y\right) \right]
\\
& \leq \exp\left(\frac{\left(\frac{1}{2\delta}\right)^2 V_{(n)}}{2\left(1-\frac{1}{2}\right)} - \frac{ y}{2\delta} \right) = \exp\left(\frac{ V_{(n)} }{4\delta^2}-\frac{ y }{2\delta} \right).
\end{align*}
Lemma~\ref{lemma_moments} leads to
\begin{align*}
 \sum_{t=1}^{n-1}\mathbb{E} \|X_t\|
 & \leq \sum_{t=1}^{n-1} \left[ \frac{G_{\epsilon}(0) }{1-\rho} + \rho^{t-1} G_{X_1}(0) \right] \\
 & \leq \frac{ (n-1) G_{\epsilon}(0) + G_{X_1}(0) }{1-\rho} =: z_{\rho,n}
\end{align*}
where we introduce the last notation for short. Now let us consider the ``favorable'' event
$$
\mathcal{E}= \left\{ \sum_{t=1}^{n-1}\|X_t\| \leq z_{\rho,n} + y \right\}
  \bigcap \left\{\sup_{f_{1:T}\in\mathcal{F}_\epsilon}|r_n(f_{1:T})-R_n(f_{1:T})| \leq x  \right\} .
$$
The previous inequalities show that
\begin{multline}
\label{proba-e}
\mathbb{P}\left( \mathcal{E}^c\right)\leq  \exp\left(\frac{ V_{(n)}}{4\delta^2}-\frac{ y }{2\delta} \right) \\
+2 \mathcal{N}(\mathcal{F},\epsilon)^T 
 \exp\left( \frac{s^2(1+\rho)^2 L^2 \mathcal{V}}{2(n-1)-2s(1+\rho)\delta L} -s x \right).
\end{multline}
On $\mathcal{E}$, we have:
\begin{align*}
R_n(\hat{f}_{1:T})
& \leq R_n(\hat{f}_{1:T}^{\epsilon}) +\epsilon L\frac{z_{\rho,n} }{n-1} \\
& \leq r_n(\hat{f}_{1:T}^{\epsilon}) + x  +\epsilon L \frac{z_{\rho,n} }{n-1}
\\
& \leq r_n(\hat{f}_{1:T}) + x + \epsilon L \left[ 2  \frac{z_{\rho,n} }{n-1}+ \frac{y}{n-1}\right]
\\
& = \min_{f_{1:T} \in\mathcal{F}_{\epsilon}^T } r_n(f_{1:T}) + x + \epsilon L   \frac{2 z_{\rho,n}+y }{n-1}
\\
& \leq \min_{f_{1:T} \in\mathcal{F}_{\epsilon}^T } R_n(f_{1:T}) + 2x + \epsilon L   \frac{2 z_{\rho,n} +y }{n-1}
\\
& \leq \min_{f_{1:T} \in\mathcal{F}^T } R_n(f_{1:T}) + 2x + \epsilon L  \frac{3 z_{\rho,n}+y }{n-1}  .
\end{align*}
In particular, the choice $\epsilon = 1/(Ln)$ ensures:
\begin{equation}
\label{almost-done}
R_n(\hat{f}_{1:T}) \leq \min_{f_{1:T} \in\mathcal{F}^T } R_n(f_{1:T}) + 2x + \frac{ 3 z_{\rho,n}  +  y }{n ( n-1) }.
\end{equation}
Fix $\eta>0$ and put:
$$ x = \frac{s(1+\rho)^2 L^2 \mathcal{V}}{2(n-1)-2s(1+\rho)\delta L} + \frac{T\mathcal{H} (\mathcal{F},\frac{1}{Ln}) + \log\left(\frac{4}{\eta}\right)}{s} $$
and
$ y =  2\delta\log\left(\frac{2}{\eta}\right) + \frac{ V_{(n)}}{2\delta }.$
Note that, plugged into~\eqref{proba-e}, these choices ensure $\mathbb{P}(\mathcal{E}^c)\leq \eta/2+\eta/2=\eta$.
Put
$$s= \frac{1}{(1+\rho)L}\sqrt{ (n-1) T\mathcal{H} (\mathcal{F},\frac{1}{Ln}) \Big/\mathcal{V} }.$$
As soon as $ 2s(1+\rho)\delta L \leq n-1$, that is actually ensured by the condition $ n \geq 1 +  4\delta^2 T \mathcal{H} (\mathcal{F},\frac{1}{Ln}) / \mathcal{V} $,
we have:
\begin{align*}
x & \leq \frac{s(1+\rho)^2 L^2 \mathcal{V}}{n-1} + \frac{T\mathcal{H} (\mathcal{F},\frac{1}{Ln}) + \log\left(\frac{4}{\eta}\right)}{s}
\\
& = 2(1+\rho) L \sqrt{ \frac{\mathcal{V}  T\mathcal{H} (\mathcal{F},\frac{1}{Ln})}{n-1}} \\
& \quad + (1+\rho)L \log\left(\frac{4}{\eta}\right) \sqrt{ \frac{\mathcal{V}}{(n-1)T\mathcal{H} (\mathcal{F},\frac{1}{Ln})}}.
\end{align*}
Pluging the expressions of $x$ and $y$ and the definition of $z_{\rho,n}$ into~\eqref{almost-done} gives:
\begin{multline*}
R_n(\hat{f}_{1:T}) \leq \min_{f_{1:T} \in\mathcal{F}^T } R_n(f_{1:T}) + 4(1+\rho) L \sqrt{ \frac{\mathcal{V} T\mathcal{H} (\mathcal{F},\frac{1}{Ln})}{n-1} }
\\
+ 2(1+\rho)L \log\left(\frac{4}{\eta}\right) \sqrt{ \frac{\mathcal{V}}{(n-1)T\mathcal{H} (\mathcal{F},\frac{1}{Ln})}}
\\
+ \frac{1}{n} \left[ 3 \frac{G_\epsilon(0) + \frac{G_{X_1}(0)}{n-1}}{1-\rho} + \frac{ 2\delta\log\left(\frac{2}{\eta}\right)  + \frac{V_{(n)} }{2\delta }}{n-1}\right]
\end{multline*}
which ends the proof.
\hfill \qed

\noindent\emph{Proof of Theorem~\ref{thm-selection}.} Fix $\epsilon>0$. For any $1\leq T\leq T_{\max}$ and $f_{1:T}=(f_1,\dots,f_T)\in\mathcal{F}^T$, chose $f_{1:T}^{\epsilon}=(f_1^\epsilon,\dots,f_T^{\epsilon})$ such that for any $i$, $\|f_{i}^{\epsilon}-f_{i}\|_{\sup}\leq \epsilon$.
Define the event
$$
\mathcal{A}= \left\{ \sum_{t=1}^{n-1}\|X_t\| \leq z_{\rho,n} + y \right\}
 \bigcap \bigcap_{T=1}^{T_{\max}} \left\{\sup_{f_{1:T}\in\mathcal{F}_\epsilon}|r_n(f_{1:T})-R_n(f_{1:T})| \leq x_{T}  \right\}
$$
where $z_{\rho,n}$ is defined as in the proof of Theorem~\ref{thm-cvg} and $x_1,\dots,x_{T_{\max}} > 0 $. We have, for any $s_1,\dots,s_{T_{\max}}< (n-1)/[L\delta(1+\rho)]$,
\begin{multline*}
\mathbb{P}\left( \mathcal{A}^c\right)\leq  \exp\left(\frac{   V_{(n)}}{4\delta^2}-\frac{y }{2\delta } \right)
\\
+2 \sum_{T=1}^{T_{\max}}\mathcal{N}(\mathcal{F},\epsilon)^T \exp\Biggl( \frac{s_T^2(1+\rho)^2 L^2 \mathcal{V}}{2(n-1)-2s_T(1+\rho)\delta L}
-s_T x_T \Biggr) \leq \eta,
\end{multline*}
the last inequality being ensured by the choice $y=2\delta\log(2/\eta) +   V_{(n)}/(2\delta)$ and, for any $T$,
\begin{align*}
x_{T} & = \frac{s_T(1+\rho)^2 L^2 \mathcal{V}}{2(n-1)-2s_T(1+\rho)\delta L} \\
& \quad \quad + \frac{T\mathcal{H} (\mathcal{F},\frac{1}{Ln}) + \log\left(4 T_{\max} /  \eta \right)}{s_T} ,\\
s_T & = \frac{1}{(1+\rho)L}\sqrt{\frac{(n-1)T\mathcal{H} (\mathcal{F},\frac{1}{Ln})}{\mathcal{V}} }.
\end{align*}
Note that this choice also leads to
$$
x_T \leq \frac{C_1}{2} \sqrt{ \frac{ T\mathcal{H} (\mathcal{F},\frac{1}{Ln})}{n-1}}
+ \frac{C_2}{2} \frac{\log\left(4 T_{\max}/\eta\right)}{\sqrt{n-1}}.
$$
On $\mathcal{A}$, we have $ R_n(\hat{f}_{1:\hat{T}}) \leq R_n(\hat{f}_{1:\hat{T}}^{\epsilon}) +\epsilon L z_{\rho,n}/(n-1)$, and
\begin{align*}
 R_n(\hat{f}_{1:\hat{T}})
 & \leq r_n(\hat{f}_{1:\hat{T}}^{\epsilon}) + x_{\hat{T}}  +\epsilon L \frac{z_{\rho,n} }{n-1} \\
& \leq r_n(\hat{f}_{1:\hat{T}}) + x_{\hat{T}}   + \epsilon L   \frac{2  z_{\rho,n}+y}{n-1}
\\
& \leq r_n(\hat{f}_{1:\hat{T}}) + \frac{C_1}{2} \sqrt{ \frac{ \hat{T}\mathcal{H} (\mathcal{F},\frac{1}{Ln})}{n-1}}
+ \frac{C_2}{2} \frac{\log\left(4 T_{\max}/\eta\right)}{\sqrt{n-1}}   + \epsilon L  \frac{ 2  z_{\rho,n}+y}{n-1}
\\
& = \min_{1\leq T\leq T_{\max}} \Biggr\{ r_n(\hat{f}_{1:T}) + \frac{C_1}{2} \sqrt{ \frac{ T \mathcal{H} (\mathcal{F},\frac{1}{Ln})}{n-1}} \Biggr\}
+ \frac{C_2}{2} \frac{\log\left(4 T_{\max}/\eta\right)}{\sqrt{n-1}}
\\ & \quad
+ \epsilon L \frac{2  z_{\rho,n}+y}{n-1}
\\
& \leq \min_{1\leq T\leq T_{\max}} \min_{f_{1:T} \in\mathcal{F}_{\epsilon}^T } \Biggl\{ r_n(f_{1:T}) + \frac{C_1}{2} \sqrt{ \frac{ T \mathcal{H} (\mathcal{F},\frac{1}{Ln})}{n-1}} \Biggr\}
+ \frac{C_2}{2} \frac{\log\left(4 T_{\max}/\eta\right)}{\sqrt{n-1}}
\\ & \quad
+ \epsilon L   \frac{2 z_{\rho,n} + y}{n-1}
\\
& \leq \min_{1\leq T\leq T_{\max}} \min_{f_{1:T} \in\mathcal{F}_{\epsilon}^T } \Biggl\{ R_n(f_{1:T}) +
C_1\sqrt{ \frac{ T \mathcal{H} (\mathcal{F},\frac{1}{Ln})}{n-1}} \Biggr\} +
\frac{C_2}{2} \frac{\log\left(4 T_{\max}/\eta\right)}{\sqrt{n-1}}
\\ & \quad
+ \epsilon L   \frac{2 z_{\rho,n} + y}{n-1}
\\
& \leq \min_{1\leq T\leq T_{\max}} \min_{f_{1:T} \in\mathcal{F}^T } \Biggl\{ R_n(f_{1:T}) + C_1\sqrt{ \frac{ T \mathcal{H} (\mathcal{F},\frac{1}{Ln})}{n-1}}\Biggr\} + \frac{C_2}{2} \frac{\log\left(4 T_{\max}/\eta\right)}{\sqrt{n-1}}  \\ & \quad
+ \epsilon L  \frac{ 3z_{\rho,n}+ y}{n-1}
\\
& \leq \inf_{1\leq T \leq T_{\max}} \inf_{f_{1:T}\in\mathcal{F}^T}
\Biggl[ R_n(f_{1:T}) +
C_1 \sqrt{ \frac{ T\mathcal{H} (\mathcal{F},\frac{1}{Ln})  }{n-1} }
+ C_2 \frac{ \log\left(4 T_{\max}/\eta\right)}{\sqrt{n-1}}
\\ & \quad
+ \frac{C_3}{n} \Biggr].\qquad\qquad\qquad\qquad\qed
\end{align*}

\bibliographystyle{apalike}
\setlength{\itemindent}{-\leftmargin}
\makeatletter\renewcommand{\@biblabel}[1]{}\makeatother


\end{document}